\def\year{2022}\relax
\documentclass[letterpaper]{article} 
\usepackage{aaai22preprint}  
\usepackage{times}  
\usepackage{helvet}  
\usepackage{courier}  
\usepackage[hyphens]{url}  
\usepackage{graphicx} 
\usepackage{natbib}  
\usepackage{caption} 
\DeclareCaptionStyle{ruled}{labelfont=normalfont,labelsep=colon,strut=off} 
\usepackage{algorithm}
\usepackage{algorithmic}
\usepackage{tcolorbox} 
\usepackage{xcolor}
\usepackage{multicol}
\usepackage{algorithm}
\usepackage{algorithmic}
\usepackage{amssymb} 
\usepackage{amsmath} 
\usepackage{graphicx}
\usepackage{multirow}
\usepackage{colortbl,booktabs} 
\usepackage{lipsum}   
\usepackage{amsthm} 
\usepackage{amsmath} 
\usepackage{arydshln} 

\newtheorem{theorem}{Theorem}[section]
\newtheorem{corollary}{Corollary}[section]
\newtheorem{lemma}{Lemma}[section]
\newtheorem{assumption}{Assumption}[section]

\newtheorem{remark}{Remark}[section]

\newcommand{\eg}{\textit{e.g.}\ }
\newcommand{\ie}{\textit{i.e.}\ }
\newcommand{\nb}{\textit{n.b.}\ }

\newcommand{\fix}[1]{\mathrm{fix}(#1)}
\newcommand{\pp}[2]{ \frac{\partial #1}{\partial #2} }
\newcommand{\dd}[2]{ \frac{\mathrm{d} #1}{\mathrm{d} #2} }
\newcommand{\itemsymbol}{{\small $\blacktriangleright$}}

\newcommand{\sD}{{\cal D}}
\newcommand{\sE}{{\cal E}}

\newcommand{\sJ}{{\cal J}}

\newcommand{\sN}{{\cal N}}

\newcommand{\sU}{{\cal U}}

\newcommand{\sY}{{\cal Y}}

\newcommand{\bbN}{{\mathbb N}} 
 
\newcommand{\bbR}{{\mathbb R}}
\newcommand{\bbE}{{\mathbb E}}

\usepackage{xcolor}
\definecolor{BLUE}{rgb}{0.3,0.3,0.9}
\definecolor{RED}{rgb}{0.8,0.05,0.05}
\definecolor{GREEN}{rgb}{0.05,0.5,0.05}
\usepackage{newfloat}
\usepackage{listings}
\floatstyle{ruled}
\newfloat{listing}{tb}{lst}{}
\floatname{listing}{Listing}

\title{JFB: Jacobian-Free Backpropagation for Implicit Networks}
\author{
 	Samy Wu Fung,\equalcontrib\textsuperscript{\rm 1}
 	Howard Heaton,\equalcontrib\textsuperscript{\rm 2}
 	Qiuwei Li,\textsuperscript{\rm 3}
 	Daniel McKenzie,\textsuperscript{\rm 4}
 	Stanley Osher,\textsuperscript{\rm 4}
 	Wotao Yin\textsuperscript{\rm 3}
     \\
}

\affiliations {
     \textsuperscript{\rm 1} Department of Applied Mathematics and Statistics, Colorado School of Mines  \\
     \textsuperscript{\rm 2} Typal Research, Typal LLC \\
     \textsuperscript{\rm 3} Alibaba Group (US), Damo Academy\\
     \textsuperscript{\rm 4} Department of Mathematics, University of California, Los Angeles \\
     swufung@mines.edu, \ 
     research@typal.llc,\ 
     li.qiuwei@alibaba-inc.com,\ 
     mckenzie@math.ucla.edu 
}

\usepackage{bibentry}

\begin{document}

\maketitle


\begin{abstract} 
    A promising trend in deep learning replaces traditional feedforward networks with implicit networks. Unlike traditional networks, implicit networks solve a fixed point equation to compute inferences. Solving for the fixed point varies in complexity, depending on provided data and an error tolerance. Importantly, implicit networks may be trained with fixed memory costs in stark contrast to feedforward networks, whose memory requirements scale linearly with depth. However, there is no free lunch --- backpropagation through implicit networks often requires solving a costly Jacobian-based equation arising from the implicit function theorem. We propose Jacobian-Free Backpropagation (JFB), a fixed-memory approach that circumvents the need to solve Jacobian-based equations.  JFB makes implicit networks faster to train and significantly easier to implement, without sacrificing test accuracy. Our experiments show implicit networks trained with JFB are competitive with feedforward networks and prior implicit networks given the same number of parameters.\footnote{All codes can be found on Github: \\{ \texttt{github.com/typal-research/jacobian\_free\_backprop}}}

\end{abstract}


\noindent A new direction has emerged from explicit to implicit neural networks  \cite{winston2020monotone,bai2019deep,bai2020multiscale,chen2018neural,ghaoui2019implicit,dupont2019ANODEs,jeon2021differentiable,zhang2020implicitly,lawrence2020almost,revay2020contracting,look2020differentiable,gould2019deep}. 
In the standard feedforward setting,  a network  prescribes a series of computations that map input data $d$ to an inference $y$. Networks can also explicitly leverage the assumption that high dimensional signals typically admit low dimensional representations in some latent space \cite{van2008visualizing,osher2017low,peyre2009manifold,elad2010role,udell2019big}. This may be done by designing the network to first map data to a latent space via a mapping $Q_\Theta$ and then apply a second mapping $S_\Theta$ to map the latent variable to the inference.
Thus, a traditional feedforward  $\sE_\Theta$ may take the compositional form
\begin{equation}
    \label{eq: explicit_network_definition}
    \sE_\Theta(d) = S_\Theta(Q_\Theta(d)),
\end{equation}
\begin{figure}[H]
    \centering
    \includegraphics[width=0.4\textwidth]{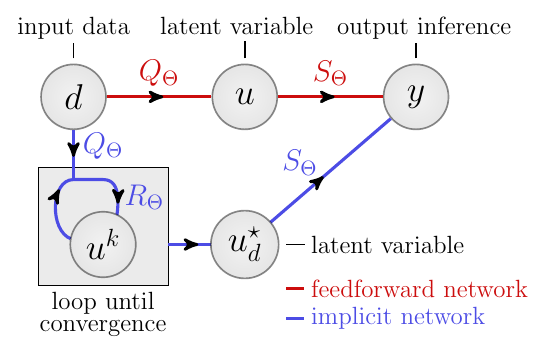} 
    \caption{Feedforward networks act by computing $S_\Theta\circ Q_\Theta$. Implicit networks add a fixed point condition using $R_\Theta$. When $R_\Theta$ is contractive (more generally: averaged) repeatedly applying $R_\Theta$ to update a latent variable $u^k$ converges to a fixed point $u^\star=R_\Theta(u^\star; Q_\Theta(d))$. 
    }
    \label{fig: explicit-implicit-comparison}
\end{figure}
\noindent which is illustrated by the red arrows in Figure \ref{fig: explicit-implicit-comparison}.
One can allow for computation in the latent space $\sU$ by introducing a self-map $R_{\Theta}(\cdot; Q_{\Theta}(d))$
and the iteration 
\begin{equation}
    u^{k+1} = R_\Theta(u^{k}; Q_\Theta(d)). \label{eq:LatentSpaceIteration}
\end{equation}
Iterating $k$ times may be viewed as a weight-tied, input-injected network, where each feedforward step applies $R_{\Theta}$ \cite{bai2019deep}. As $k \to \infty$, {\em i.e.} the latent space portion becomes deeper,  the limit of \eqref{eq:LatentSpaceIteration} yields a {\em fixed point equation}. Implicit networks capture this ``infinite depth'' behaviour by using $R_{\Theta}(\cdot \ ; Q_{\Theta}(d))$ to define a fixed point condition rather than an explicit computation:
\begin{equation}
    \label{eq: implicit-model-definition}
    \sN_\Theta(d) \triangleq S_\Theta(u_d^\star)
    \ \ \ \mbox{where} \ \ \ 
    u_d^\star = R_\Theta(u_d^\star, Q_\Theta(d)),
\end{equation}
as shown by blue in Figure~\ref{fig: explicit-implicit-comparison}. Special cases of the network in \eqref{eq: implicit-model-definition} recover architectures introduced in prior works:
\begin{itemize}
    \item[$\vartriangleright$] Taking $S_{\Theta}$ to be the identity recovers the well-known Deep Equilibrium Model (DEQ) \cite{bai2019deep,bai2020multiscale}.
    \item[$\vartriangleright$] Choosing $S_{\Theta}$ as the identity, $Q_{\Theta}$ to be an affine map and $R_{\Theta}(u,Q_{\Theta}(d)) = \sigma(Wu + Q_{\Theta}(d))$ yields Monotone Operator Networks \cite{winston2020monotone} as long as $W$ and $\sigma$ satisfy  additional conditions. Allowing $S_{\Theta}$ to be linear yields the model proposed in \cite{ghaoui2019implicit}.
\end{itemize}
Three immediate questions arise from \eqref{eq: implicit-model-definition}:
\begin{itemize}
    \item[\itemsymbol] Is the definition in \eqref{eq: implicit-model-definition} well-posed?
    \item[\itemsymbol] How is $\sN_{\Theta}(d)$ evaluated?
    \item[\itemsymbol] How are the weights $\Theta$ of $\sN_{\Theta}$ updated during training?
\end{itemize}
Since the first two points are well-established~\cite{winston2020monotone,bai2019deep}, we briefly review these in Section~\ref{sec: fixed-point-framework} and focus on the third point. Using gradient-based methods for training  
requires computing $d\sN_{\Theta}\big/d\Theta$, and in particular, $du_d^\star\big/d\Theta$. Hitherto,   previous works   computed $du_d^\star\big/d\Theta$ by solving a Jacobian-based equation (see Section~\ref{sec: backpropagation}). Solving this linear system is computationally expensive and prone to instability, particularly when the dimension of the latent space is large and/or includes certain structures (\eg batch normalization and/or dropout)~\cite{bai2019deep,bai2020multiscale}. 

Our primary contribution is a new and simple {\bf Jacobian-Free Backpropagation} (JFB) technique for training implicit networks that avoids {\em any} linear system solves. 
Instead, our scheme backpropagates by omitting the Jacobian term, resulting in a form of preconditioned gradient descent. JFB yields much faster training of implicit networks and allows for a wider array of architectures.

\section{Why Implicit Networks?} \label{sec: why-implicit}
Below, we discuss several advantages of implicit networks over explicit, feedforward networks. 

\paragraph{Implicit networks for implicitly defined outputs} In some applications, the desired network output is most aptly described implicitly as a fixed point, not via an explicit function. As a toy example, consider predicting the variable $y\in\bbR$ given $d \in [-1/2,1/2]$  when $(d,y)$ is known to satisfy  
\begin{equation}
\begin{aligned}
    \label{eq: toy_manifold}
    y = d + y^5.
\end{aligned}
\end{equation}
Using $y_1 = 0$ and the iteration
\begin{equation}
    \label{eq: toy_Tk}
    y_{k+1} = T(y_k;d) \triangleq d + y_k^5, \ \ \ \mbox{for all $k\in\bbN$,}
\end{equation}
one obtains $y_k\rightarrow y$. 
In this setting, $y$ is exactly (and implicitly) characterized by $y = T(y,d)$. On the other hand, an explicit solution to~\eqref{eq: toy_manifold} requires an infinite series representation, unlike the simple formula $T(y,d) = d + y^5$.
See Appendix~\ref{app: toy_example} for further details.
Thus, it can be simpler and more appropriate to model a relationship implicitly.
For example, in areas as diverse as game theory and inverse problems, the output of interest may naturally be characterized as the fixed point to an operator parameterized by the input data $d$. Since implicit networks find fixed points by design, they are well-suited to such problems as shown by   recent works \cite{heaton2021feasibility,heaton2021learn,gilton2021deep}.

\paragraph{``Infinite depth'' with constant memory training} 
As mentioned, solving for the fixed point of $R_{\Theta}(\cdot\ ;Q_{\Theta}(d))$ is analogous to a forward pass through an ``infinite depth'' (in practice, very deep) weight-tied, input injected feedforward network. However, implicit networks do not need to store intermediate quantities of the forward pass  
for backpropagation. 
Consequently, implicit networks are trained using \textit{constant memory costs} with respect to depth -- relieving a major bottleneck of training deep networks.

\paragraph{No loss of expressiveness} 
Implicit networks as defined in \eqref{eq: implicit-model-definition} are at least as expressive as feedforward networks. This can easily be observed by setting $R_\Theta$ to simply return $Q_\Theta$; in this case, the implicit  $\sN_\Theta$ reduces to the feedforward  $\sE_\Theta$ in~\eqref{eq: explicit_network_definition}. More interestingly, the class of implicit networks in which $S_{\Theta}$ and $Q_{\Theta}$ are constrained to be affine maps contains all feedforward networks, and is thus at least as expressive \cite{ghaoui2019implicit}, \cite[Theorem 3]{bai2019deep}. Universal approximation properties of implicit networks then follow immediately from such properties of conventional deep neural models (\eg see \cite{csaji2001approximation,lu2017expressive,kidger2020universal}). \\

\noindent We also mention a couple limitations of implicit networks. 

\paragraph{Architectural limitations} As discussed above, in theory given any feedforward network one may write down an implicit network yielding the same output (for all inputs). In practice, evaluating the implicit network requires finding a fixed point of $R_{\Theta}$. The fixed point finding algorithm then places constraints on $R_{\Theta}$ (\eg Assumption~\ref{ass: T-contraction}). Guaranteeing the existence and computability of $d\sN_{\Theta}\big/d\Theta$ places further constraints on $R_{\Theta}$. For example, if Jacobian-based backpropagation is used, $R_{\Theta}$ cannot contain batch normalization \cite{bai2019deep}.

\paragraph{Slower inference} Once trained, inference with an implicit network requires solving for a fixed point of $R_{\Theta}$. Finding this fixed point using an iterative algorithm requires evaluating $R_{\Theta}$ repeatedly and, thus, is often slower than inference with a feedforward network.

\section{Implicit Network Formulation} \label{sec: fixed-point-framework}
All terms  presented in this section are provided in a general context, which is later made concrete  for each application. We include a subscript $\Theta$ on various terms to emphasize   the indicated mapping will ultimately be parameterized in terms of tunable weights\footnote{We use the same subscript for all terms, noting each operator typically depends on a portion of the weights.} $\Theta$.
At the highest level, we are interested in constructing a neural network $\sN_\Theta:\sD\rightarrow\sY$ that maps from a data space\footnote{Each space is assumed to be a real-valued finite dimensional Hilbert space (\eg $\mathbb{R}^n$) endowed with a product $\left<\cdot,\cdot\right>$ and norm $\|\cdot\|$. It will be clear from context which space is being used.} $\sD$ to an inference space $\sY$. The implicit portion of the network uses a latent space $\sU$, and data is mapped to this latent space by $Q_\Theta\colon \sD \to \sU$. We define the \textit{network operator}  $T_\Theta: \sU\times\sD\rightarrow\sU$ by 
\begin{equation}
   T_\Theta(u; d) \triangleq R_\Theta(u, Q_\Theta(d)).
   \label{eq: T-composition-R-Q}
\end{equation}
Provided input data $d$, our aim is to find the unique fixed point $u_d^\star$ of $T_\Theta(\cdot\ ; d)$ and then map $u_d^\star$ to the inference space $\sY$ via a final mapping $S_\Theta:\sU\rightarrow\sY$. This enables us to define an implicit network $\sN_\Theta$ by
\begin{equation}
    \sN_\Theta(d) \triangleq S_\Theta(u_d^\star)
    \ \ \ \mbox{where} \ \ \  
    u_d^\star = T_\Theta(u_d^\star; d).
    \label{eq: FPN-fixed-point-def}
\end{equation}
\begin{minipage}{0.45\textwidth}
\begin{algorithm}[H]
\caption{Implicit Network  with Fixed Point Iteration}
\label{alg: FPN_Abstract}
\begin{algorithmic}[1]           
    \STATE{\begin{tabular}{p{0.475\textwidth}r}
     \hspace*{-8pt}  $\sN_\Theta(d)\colon$
     &  
     $\vartriangleleft$ Input data is $d$
     \end{tabular}}    
    
    \STATE{\begin{tabular}{p{0.475\textwidth}r}
     \hspace*{0pt} $u^1 \leftarrow \hat{u}$
     &  
     $\vartriangleleft$ Assign latent term
     \end{tabular}}        

    \STATE{\begin{tabular}{p{0.475\textwidth}r}
     \hspace*{-1pt} {\bf while} $\|u^k - T_\Theta(u^k;d)\| > \varepsilon$\hspace*{-20pt}\ 
     &  
     $\vartriangleleft$ Loop til converge
     \end{tabular}}  
     
    \STATE{\begin{tabular}{p{0.475\textwidth}r}
     \hspace*{3pt} $u^{k+1} \leftarrow T_\Theta(u^k;d)$
     &  
     $\vartriangleleft$ Refine latent term
     \end{tabular}}      

    \STATE{\begin{tabular}{p{0.475\textwidth}r}
     \hspace*{3pt} $k\leftarrow k+1$
     &  
     $\vartriangleleft$ Increment counter
     \end{tabular}}       

    \STATE{\begin{tabular}{p{0.475\textwidth}r}
     \hspace*{-8pt} {\bf return} $S_\Theta(u^k)$  
     &  
     $\vartriangleleft$ Output \textit{estimate}
     \end{tabular}}   
\end{algorithmic}
\end{algorithm}  
\vspace{0pt}
\end{minipage}
Implementation considerations for $T_\Theta$ are discussed below. We also introduce  assumptions on $T_\Theta$ that yield sufficient conditions to use the simple procedure in Algorithm \ref{alg: FPN_Abstract} to approximate $\sN_\Theta(d)$. In this algorithm, the latent variable initialization $\hat{u}$ can be any fixed quantity (\eg the zero vector). The inequality in Step 3 gives a fixed point residual condition that measures convergence. Step 4 implements a fixed point update. The estimate of the inference $\sN_\Theta(d)$ is computed by applying $S_\Theta$ to the latent variable $u^k$ in Step 6.  The blue path in Figure \ref{fig: explicit-implicit-comparison} visually summarizes Algorithm~\ref{alg: FPN_Abstract}.

\begin{figure*}
    \centering
    \includegraphics[width=1\textwidth]{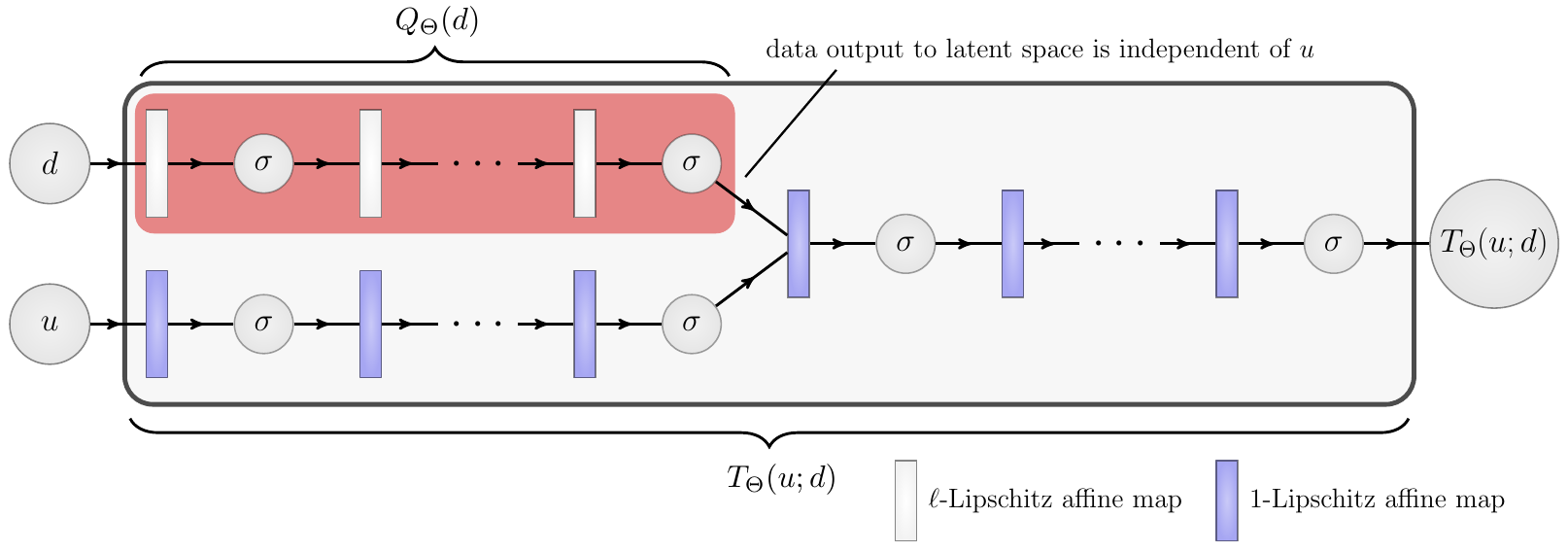}
    \caption{Diagram of a possible architecture for network operator $T_\Theta$ (in large rectangle). Data $d$ and latent $u$ variables are processed in two streams by nonlinearities (denoted by $\sigma$) and affine mappings (denoted by rectangles). These streams merge into a final stream that may also contain transformations. Light gray and blue affine maps are   $\ell$-Lipschitz and 1-Lipschitz, respectively. The mapping $Q_\Theta$ from data space to latent space is enclosed by the red rectangle.  
    }
    \label{fig: data-streams}
\end{figure*} 

\paragraph{Convergence} Finitely many loops in Steps 3 and 4 of Algorithm \ref{alg: FPN_Abstract} is guaranteed by a classic functional analysis result \cite{banach1922operations}. This approach is used by several implicit networks~\cite{ghaoui2019implicit,winston2020monotone,jeon2021differentiable}. 
Below we present a variation of Banach's result for our setting.

\begin{assumption} \label{ass: T-contraction}
    The mapping $T_\Theta$ is $L$-Lipschitz with respect to its inputs $(u,d)$, \ie,
    \begin{equation}
        \begin{split}
        \| &T_\Theta(u;\ d) - T_\Theta(v;\ w)\| \leq L \| (u,d) - (v,w)\|, 
        \end{split}
    \end{equation}
    for all $(u,d),(v,w)\in\sU\times\sD$.
    Holding $d$ fixed, the operator $T_\Theta(\cdot; d)$ is a contraction, \ie  there exists $\gamma \in [0,1)$ such that
    \begin{equation}
        \|T_\Theta(u;d) - T_\Theta(v;d)\| \leq \gamma \| u-v\|,
        \ \ \ \mbox{for all $u,v\in\sU$.}
        \label{eq: T-u-gamma-contraction}
    \end{equation}
\end{assumption}

\begin{remark}
    The $L$-Lipschitz condition on $T_\Theta$  is used since recent works show Lipschitz continuity with respect to inputs improves generalization \cite{sokolic2017robust,Gouk2021Regularisation,finlay2018lipschitz} and adversarial robustness \cite{cisse2017parseval,anil2019sorting}. 
\end{remark}

\begin{theorem} {\sc (Banach)}
    For any $u^1 \in \sU$, if the sequence $\{u^k\}$ is generated via the update relation
    \begin{equation}
        u^{k+1} = T_\Theta(u^k;\ d), \ \ \ \mbox{for all $k\in\bbN$},
    \end{equation}
    and if Assumption \ref{ass: T-contraction} holds, then $\{u^k\}$ converges linearly to the unique fixed point $u_d^\star$ of $T_\Theta(\cdot;d)$. 
\end{theorem}

\paragraph{Alternative Approaches} In \cite{bai2019deep, bai2020multiscale} Broyden's method is used for finding $u_d^\star$. Broyden's method is a quasi-Newton scheme and so at each iteration it updates a stored approximation to the Jacobian $J_k$ and then solves a linear system in $J_k$. Since in this work our goal is to explore truly {\em Jacobian-free} approaches, we stick to the simpler fixed point iteration scheme when computing $\tilde{u}$ ({\em i.e.} Algorithm~\ref{alg: FPN_Abstract}). In the contemporaneous \cite{gilton2021deep}, it is reported that using fixed point iteration in conjunction with Anderson acceleration finds $\tilde{u}$ faster than both vanilla fixed point iteration and Broyden's method. Combining JFB with Anderson accelerated fixed point iteration is a promising research direction we leave for future work. 

\paragraph{Other Implicit Formulations} A related implicit learning formulation is the well-known neural ODE model~\cite{chen2018neural,dupont2019ANODEs,ruthotto2021introduction}.
Neural ODEs leverage known connections between deep residual models and discretizations of differential equations~\cite{haber2017stable,weinan2017proposal,ruthotto2019deep,chang2018reversible,finlay2020train,lu2018beyond}, and replace these discretizations by black-box ODE solvers in forward and backward passes.
The implicit property of these models arise from their method for computing gradients.
Rather than backpropagate through each layer, backpropagation is instead done by solving the adjoint equation~\cite{jameson1988aerodynamic} using a blackbox ODE solver as well. This is analogous to solving the Jacobian-based equation when performing backpropagation for implicit networks (see \eqref{eq: adjoint_equation}) and allows the user to alleviate the memory costs of backpropagation through deep neural models by solving the adjoint equation at additional computational costs. A drawback is that the adjoint equation must be solved to high-accuracy; otherwise, a descent direction is not necessarily guaranteed~\cite{gholami2019anode,onken2020discretize,onken2021ot}.  

\section{Backpropagation} \label{sec: backpropagation}
We present a simple way to backpropagate with implicit networks, called Jacobian-free backprop (JFB). Traditional backpropagation will \textit{not} work effectively for implicit networks since  forward propagation during training could   entail hundreds or thousands of iterations, requiring ever growing memory to store computational graphs. 
On the other hand, implicit models maintain fixed memory costs by backpropagating ``through the fixed point'' and solving a Jacobian-based equation (at potentially substantial added computational costs).
The key step to circumvent this Jacobian-based equation with JFB is to tune weights by  using a preconditioned gradient.   
Let $\ell:\sY \times \sY \rightarrow \bbR$ be a smooth loss function, denoted by $\ell(x,y)$, and  consider the  training  problem 
\begin{equation} 
    \min_{\Theta} \bbE_{d\sim\sD} \big[ \ell\left( y_d, \sN_\Theta(d) \right) \big],
    \label{eq: training-problem}
\end{equation}
where we abusively write $\sD$ to also mean a distribution.
For clarity of presentation, in the remainder of this section  we notationally suppress the dependencies on  weights $\Theta$ by letting $u_d^\star$ denote the fixed point in (\ref{eq: FPN-fixed-point-def}). Unless noted otherwise, mapping arguments are implicit in this section; in each implicit case, this will correspond to   entries in (\ref{eq: FPN-fixed-point-def}). 
We begin with   standard assumptions enabling us to differentiate   $\sN_\Theta$.

\begin{assumption} \label{ass: smooth-mappings}
    The   mappings $S_\Theta$ and  $T_\Theta$ are continuously differentiable with respect to $u$ and $\Theta$. 
\end{assumption}

\begin{assumption} \label{ass: independent-theta}
    The weights $\Theta$ may be written as a tuple $\Theta = (\theta_S, \theta_T)$ such that weight paramaterization of  $S_\Theta$ and $T_\Theta$ depend only on $\theta_S$ and $\theta_T$, respectively.\footnote{This assumption is easy to ensure in practice. For notational brevity, we use the subscript $\Theta$ throughout.} 
\end{assumption}

Let $\sJ_\Theta$ be defined as the identity operator, denoted by $\mathrm{I}$, minus the Jacobian\footnote{Under Assumption \ref{ass: T-contraction}, the Jacobian $\sJ_\Theta$ exists almost everywhere. However,   presentation is   cleaner by assuming smoothness.} of $T_\Theta$ at $(u,d)$, \ie
\begin{equation}
    \label{eq: Jacobian_def}
    \sJ_\Theta(u;d)   \triangleq \mathrm{I} - \dd{T_\Theta}{u}(u;d).
\end{equation} 
Following 
\cite{winston2020monotone,bai2019deep}, we differentiate both sides of the fixed point relation in (\ref{eq: FPN-fixed-point-def})  to obtain, by the implicit function theorem,
\begin{equation}
    \dfrac{\mathrm{d} u_d^\star }{\mathrm{d} \Theta}
    = \dfrac{\partial T_\Theta}{\partial u} \dfrac{\mathrm{d} u_d^\star}{\mathrm{d}\Theta}
    + \dfrac{\partial T_\Theta}{\partial \Theta}
    \ \ \ \Longrightarrow \ \ \ 
    \dfrac{\mathrm{d} u_d^\star }{\mathrm{d} \Theta}
    = \sJ_\Theta^{-1} \cdot \pp{T_\Theta}{\Theta},
    \label{eq: adjoint_equation}
\end{equation}  
where $\sJ_\Theta^{-1}$ exists whenever $\sJ_\Theta$ exists (see Lemma \ref{lemma: J-coercive}). Using the chain rule gives the   loss gradient
\begin{equation}
    \begin{split}
        \dd{}{\Theta}\left[ \ell(y_d, \sN_\Theta(d)) \right]
        &= \dd{}{\Theta}\Big[ \ell(y_d, S_\Theta(T_\Theta(u_d^\star,d))\Big]
        \\
        &=  \pp{\ell}{y} \left[ \dd{S_\Theta}{u} \sJ_\Theta^{-1} \pp{T_\Theta}{\Theta} + \pp{S_\Theta}{\Theta}\right].
    \end{split}
    \label{eq: standard-implicit-gradient}
\end{equation}
The matrix  $\sJ_\Theta$ satisfies the   inequality (see Lemma \ref{lemma: J-coercive}) 
\begin{equation}
    \left< u, \sJ_\Theta^{-1}u\right> \geq \dfrac{1-\gamma}{(1+\gamma)^2}\|u\|^2,\  \ \ \mbox{for all $u\in\sU$.}
\end{equation}
Intuitively, this coercivity property makes it seem possible  to remove $\sJ_\Theta^{-1}$ from   (\ref{eq: standard-implicit-gradient}) and   backpropagate using
\begin{equation}
    \begin{split}
    p_\Theta &\triangleq  - \dd{}{\Theta}\Big[ \ell( y_d, S_\Theta(T_\Theta(u, d)) \Big]_{u=u_d^\star} 
    \\
    &=-  \pp{\ell}{y} \left[ \dd{S_\Theta}{u}   \pp{T_\Theta}{\Theta} + \pp{S_\Theta}{\Theta}\right].
    \end{split}
\end{equation}

The omission of $\sJ_\Theta^{-1}$ admits two straightforward interpretations. 
Note  $\sN_\Theta(d) = S_\Theta(T_\Theta(u_d^\star;d ))$, and so $p_\Theta$ is precisely the gradient of the expression
$
    \ell(y_d, S_\Theta(T_\Theta(u_d^\star;d ))),
$
treating $u_d^\star$ as a constant \textit{independent} of $\Theta$. The distinction is that using    $S_\Theta(T_\Theta(u_d^\star;d ))$ assumes, perhaps by chance, the user chose the first iterate $u^1$ in their fixed point iteration (see Algorithm \ref{alg: FPN_Abstract}) to be precisely the fixed point $u_d^\star$. This makes the iteration trivial, ``converging'' in one iteration. We can simulate this behavior by using the fixed point iteration to find $u_d^\star$ and   only backpropagating through the final step of the fixed point iteration, as shown in Figure \ref{fig: backprop}.

Since the weights $\Theta$ typically lie in a space of much higher dimension than the latent space $\sU$, the Jacobians $\partial S_\Theta / \partial \Theta$ and $\partial T_\Theta / \partial \Theta$   effectively always have full column rank. 
We leverage this fact via the following assumption.  

\begin{assumption} \label{ass: good-conditioning}
    Under Assumption \ref{ass: independent-theta}, given any weights $\Theta = (\theta_S, \theta_T)$ and data $d$, the matrix
    \begin{equation}
        M \triangleq \  \left[ \begin{array}{cc} \pp{S_\Theta}{\theta_S}  & 0 \\ 0 & \pp{T_\Theta}{\theta_T} \end{array}\right]
        \label{eq: M-def-gradient}
    \end{equation}
    has full column rank and is sufficiently well conditioned to satisfy the inequality\footnote{The term $\gamma$ here refers to the contraction factor in (\ref{eq: T-u-gamma-contraction}).}
    \begin{equation}
        \kappa( M^\top M) = \dfrac{\lambda_{\mathrm{max}}(M^\top M)}{\lambda_{\mathrm{min}}(M^\top M)} \leq \dfrac{1}{\gamma}.
        \label{eq: gradient-conditioning}
    \end{equation}
\end{assumption}

\begin{remark}
    The conditioning portion of the above assumption  is useful for bounding the worst-case behavior in our analysis. However, we found it unnecessary to enforce this in our experiments  for effective training (\eg see Figure \ref{fig: comparison-CIFAR10}), which we hypothesize is justified because   worst case behavior rarely occurs in practice and we train using   averages of $p_\Theta$ for samples drawn from   large data sets.
\end{remark}

Assumption \ref{ass: good-conditioning} gives rise to a second interpretation of JFB. Namely, 
the full column rank of $M$ enables us to rewrite $p_\Theta$ as a preconditioned gradient, \ie 
\begin{equation}
    p_\Theta = \underbrace{\left( M \left[ \begin{array}{cc} \mathrm{I} & 0 \\ 0 & \sJ_\Theta \end{array}\right] M^+ \right)}_{\mbox{preconditioning term}} \dd{\ell}{\Theta},
\end{equation}
where $M^+$ is the Moore-Penrose pseudo inverse \cite{moore1920reciprocal,penrose1955generalized}.
These insights lead to our main result.

\begin{theorem} \label{thm: backprop-descent}    
 If  Assumptions \ref{ass: T-contraction}, \ref{ass: smooth-mappings}, \ref{ass: independent-theta}, and \ref{ass: good-conditioning} hold for given weights $\Theta$ and data $d$,    then
    \begin{equation}
        p_\Theta \triangleq 
        - \dd{}{\Theta}\Big[ \ell(y_d, S_\Theta(T_\Theta(u,d)) \Big]_{u=u_d^\star} 
        \label{eq: descent-direction}
    \end{equation} 
    is a descent direction for $\ell(y_d,\sN_\Theta(d))$ with respect to $\Theta$.
\end{theorem} 

Theorem \ref{thm: backprop-descent} shows we can avoid difficult computations associated with $\sJ_\Theta^{-1}$ in (\ref{eq: standard-implicit-gradient}) (\ie solving an associated linear system/adjoint equation) in implicit network literature \cite{chen2018neural,dupont2019ANODEs,bai2019deep,winston2020monotone}. Thus, our scheme more naturally applies to general multilayered $T_\Theta$ and is substantially simpler to code. 
Our    scheme is juxtaposed in Figure \ref{fig: backprop} with classic and Jacobian-based schemes.

Two additional considerations must be made when determining the efficacy of training a model using (\ref{eq: descent-direction}) rather than Jacobian-based gradients  (\ref{eq: standard-implicit-gradient}).

\begin{minipage}{0.45\textwidth}
\vspace{2pt}
\begin{itemize}
    \item[\itemsymbol]  Does use of $p_\Theta$ in (\ref{eq: descent-direction}) degrade training/testing performance relative to (\ref{eq: standard-implicit-gradient})?
    
    \item[\itemsymbol] Is the term $p_\Theta$ in (\ref{eq: descent-direction}) resilient to errors in estimates of the fixed point $u_d^\star$?
\end{itemize}
\vspace{2pt}
\end{minipage}

The first answer is   our training scheme   takes a different path to minimizers than using gradients with the implicit model. Thus, for nonconvex problems, one should not expect the results to be the same.  
In our experiments in Section \ref{sec: experiments}, using (\ref{eq: descent-direction})  is competitive (\ref{eq: standard-implicit-gradient})  for all tests (when applied to nearly identical models).  
The second inquiry is partly answered by the corollary below, which states JFB yields descent even for approximate fixed points. 

\begin{corollary} \label{cor: perturbation-resilience}
     Given weights $\Theta$ and data $d$, there exists $\varepsilon > 0$ such that
if $u_d^\varepsilon\in\sU$ satisfies $\|u_d^\varepsilon- u_d^\star\| \leq \varepsilon$ and the assumptions of Theorem \ref{thm: backprop-descent} hold, then  
    \begin{equation}
        p_\Theta^\varepsilon \triangleq 
        - \dd{}{\Theta}\Big[ \ell(y_d, S_\Theta(T_\Theta(u,d))\Big]_{u=u_d^\varepsilon}
    \end{equation} 
    is a descent direction of $\ell(y_d, \sN_\Theta(d))$ with respect to $\Theta$.  
\end{corollary}

We are not aware of any analogous results for error tolerances in the implicit depth literature.
 
\paragraph{Coding Backpropagation} 
A key feature of JFB is its simplicity of implementation. In particular, the backpropagation of our scheme is   similar to that of a standard backpropagation. We illustrate this in the   sample of  PyTorch~\cite{paszke2017automatic}  code in Figure \ref{fig: pytorch-sample-code}.
Here { \verb|explicit_model|} represents $S_\Theta(T_\Theta(u;d))$. The fixed point { $u_d^\star=$ \verb|u_fxd_pt|} is computed by successively applying $T_\Theta$ (see Algorithm \ref{alg: FPN_Abstract})  within a  { \verb|torch.no_grad()|} block.
With this fixed point, { \verb|explicit_model|}  evaluates and returns $S_\Theta ( T_\Theta(u_d^\star, d))$ to { \verb|y|} in { \verb|train|} mode  (to create the computational graph). Thus, our   scheme coincides with standard backpropagation through an explicit model with \emph{one} latent space layer.
On the other hand, standard implicit models backpropagate by solving a linear system to apply $\sJ_\Theta^{-1}$ as in (\ref{eq: standard-implicit-gradient}). That approach   requires users to manually update the parameters,  use more computational resources, and make considerations (\eg conditioning of $\sJ_\Theta^{-1}$) for each   architecture used.
\begin{figure}[H]
\begin{tcolorbox}[colback=white!96!black,colframe=black!25!white,title=Implicit Forward + Proposed Backprop, coltitle=black, left=1mm]
\vspace*{-5pt}
\begin{verbatim}  
u_fxd_pt = find_fixed_point(d)
y = explicit_model(u_fxd_pt, d)
loss = criterion(y, labels)
loss.backward()
optimizer.step()
\end{verbatim} 
\end{tcolorbox}
    \caption{Sample PyTorch code for backpropagation} 
    \label{fig: pytorch-sample-code}
\end{figure}


\paragraph{Neumann Backpropagation} 
The inverse of the Jacobian in~\eqref{eq: Jacobian_def} can be expanded using a Neumann series, \ie
\begin{equation}
    \sJ_\Theta^{-1} = \left( \mbox{I} - \dd{T_\Theta}{u}\right)^{-1} = \sum_{k=0}^\infty  \left( \dd{T_\Theta}{u}\right)^k.
    \label{eq: neumann-series}
\end{equation} 
Thus, JFB is a zeroth-order approximation to the Neumann series. 
In particular, JFB resembles the Neumann-RBP approach for recurrent networks~\cite{liao2018reviving}. 
However, Neumann-RBP does not guarantee a descent direction or guidelines on how to truncate the Neumann series. This is generally difficult to achieve in theory and practice~\cite{aicher2020adaptively}.
Our work differs from~\cite{liao2018reviving} in that we focus purely on implicit networks, prove descent guarantees for JFB, and provide simple PyTorch implementations.
Similar approaches exist in hyperparameter optimization, where truncated Neumann series are is used to approximate second-order updates during training~\cite{luketina2016scalable,lorraine2020optimizing}. Finally, similar zeroth-order truncations of the Neumann series have been employed, albeit without proof, in Meta-learning \cite{finn2017model,finn2019meta} and in training transformers \cite{geng2021is}.

\begin{figure*}[t]
    \centering
    \includegraphics[]{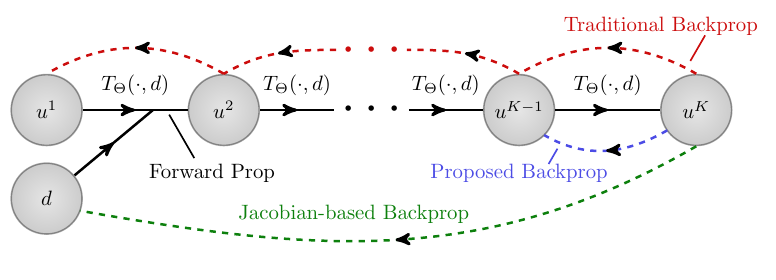} 
    \caption{Diagram of backpropagation schemes for recurrent implicit depth models. Forward propagation is tracked via solid arrows point to the right (\nb each forward step uses $d$). Backpropagation is shown via dashed arrows pointing to the left. Traditional backpropagation requires memory capacity proportional to depth (which is implausible for large $K$). Jacobian-based backpropagation solves an associated equation dependent upon the data $d$ and operator $T_\Theta$. JFB uses a single backward step, which avoids both large memory capacity requirements and solving a Jacobian-type equation.}
    \label{fig: backprop}
\end{figure*}

\section{Experiments} \label{sec: experiments} 
 
This section shows the effectiveness of JFB  using PyTorch \cite{paszke2017automatic}.
All networks are ResNet-based such that Assumption \ref{ass: independent-theta} holds.\footnote{A weaker version of Assumption \ref{ass: smooth-mappings} also holds in practice, \ie differentiability almost everywhere.} 
One can ensure Assumption \ref{ass: T-contraction} holds (\eg via spectral normalization). Yet, in our experiments we found this unnecessary since tuning the weights automatically encouraged contractive behavior.\footnote{We found (\ref{eq: T-u-gamma-contraction}) held for batches of  data during training, even when using batch normalization. See Appendix~\ref{app: experimental-settings} for more details.}
All experiments are run on a single NVIDIA TITAN X GPU with 12GB RAM. Further details  are in Appendix~\ref{app: experimental-settings}.
 
\begin{table}[H]
    \small
    \centering
    \def\ROWCOLOR{black!10!white}
    \begin{tabular}{l r r}
	    \multicolumn{3}{c}{\textbf{MNIST}}  \\ \midrule
		{Method}      & {Network size} & {Acc.} \\\midrule 
		
		\rowcolor{\ROWCOLOR}
		{Explicit} & {54K} & {99.4\%} \\  
		Neural ODE$^{\dagger}$ &  84K & 96.4\% \\
		\rowcolor{\ROWCOLOR}
		Aug. Neural ODE$^{\dagger}$ &  84K & 98.2\%  \\	 
		MON $^\ddag$ & 84K & {99.2\%}\\  
		\rowcolor{\ROWCOLOR}
		\textbf{JFB-trained Implicit ResNet (ours)} & {54K} & \textbf{99.4}\% 
		\\
		\vspace{-5pt}
		\\		    
	    \multicolumn{3}{c}{\textbf{SVHN}}\\ \midrule
		{Method}      & {Network size} & {Acc.}\\  \midrule
		\rowcolor{\ROWCOLOR}
		Explicit      & 164K & 93.7\% \\ 
		Neural ODE$^{\dagger}$& 172K & 81.0\% \\
		\rowcolor{\ROWCOLOR}
		Aug. Neural ODE$^{\dagger}$& 172K & 83.5\%      \\ 
		MON (Multi-tier lg)$^\ddag$ & 170K & {92.3\%}  \\
		\rowcolor{\ROWCOLOR}
		\textbf{JFB-trained Implicit ResNet (ours)} & 164K & \textbf{94.1\%}  
		\\
		\\
        \multicolumn{3}{c}{\textbf{CIFAR-10}}		\\ \midrule
		{Method}   & {Network size} & {Acc.}\\ \midrule

		
	    
	
		\rowcolor{\ROWCOLOR}
        Explicit (ResNet-56)$^*$                  & 0.85M & 93.0\% 
		\\
		MON (Multi-tier lg){$^\ddag$}{$^*$}  & 1.01M & {89.7\%} 
		\\
		\rowcolor{\ROWCOLOR}
        {\bf JFB-trained Implicit ResNet (ours)$^*$} & 0.84M  &  \textbf{93.7\%} 
        \\		
        \ \\[-7pt]
        \hdashline[3.5pt/8pt] \\[-8pt]
        
		Multiscale DEQ{$^*$} & 10M &  93.8\%
	\end{tabular}
    \vspace{5pt}
    \caption{Test accuracy of   JFB-trained Implicit ResNet   compared to Neural ODEs, Augmented NODEs, and MONs;  $^\dagger$as reported in \cite{dupont2019ANODEs}; $^\ddag$as reported in \cite{winston2020monotone}; *with data augmentation}
	\label{tab: results}
\end{table}

\begin{table*}[t]
    \def\CCOL{black!10!white} 
    \centering
    \begin{tabular}{c c  c c c}
        & Dataset &  Avg time per epoch (s) & \# of $\mathcal{J}$ mat-vec products & Accuracy \%
        \\
        \toprule
        \multirow{3}{*}{\begin{minipage}{0.8in} \centering Jacobian based \end{minipage}}
        & \cellcolor{\CCOL} MNIST        
        &  \cellcolor{\CCOL} 28.4
        & \cellcolor{\CCOL} $6.0 \times 10^6$
        & \cellcolor{\CCOL} 99.2
        \\
        & SVHN       & $\:$92.8 & $\:$ $1.4 \times 10^{7}$ & $\:$90.1
        \\
        & \cellcolor{\CCOL} CIFAR10   
        & \cellcolor{\CCOL} 530.9
        & \cellcolor{\CCOL} $9.7 \times 10^8$
        & \cellcolor{\CCOL} 87.9\\
        \toprule
        \multirow{3}{*}{\begin{minipage}{0.8in} \centering JFB \end{minipage}}
        & \cellcolor{\CCOL} MNIST        
        & \cellcolor{\CCOL} 17.6
        & \cellcolor{\CCOL} 0 
        & \cellcolor{\CCOL} 99.4 
        \\
        & SVHN       & $\:$36.9 & $\:$0  & $\:$94.1
        \\
        & \cellcolor{\CCOL} CIFAR10   
        & \cellcolor{\CCOL} {146.6}
        & \cellcolor{\CCOL} 0
        & \cellcolor{\CCOL} 93.67
        \\ 
    \end{tabular}
    \caption{ Comparison of Jacobian-based backpropagation (first three rows) and our proposed JFB approach. ``Mat-vecs'' denotes matrix-vector products.}
    \label{tab: adjointComparison}
\end{table*}

\subsection{Classification} 
\label{subsec: classification}
We train implicit networks on three benchmark image classification datasets licensed under CC-BY-SA: SVHN~\cite{netzer2011reading}, MNIST~\cite{lecun2010mnist}, and CIFAR-10~\cite{krizhevsky2009learning}. 
 Table~\ref{tab: results} compares our results with state-of-the-art results for implicit networks, including Neural ODEs~\cite{chen2018neural}, Augmented Neural ODEs~\cite{dupont2019ANODEs}, Multiscale DEQs~\cite{bai2020multiscale}, and MONs~\cite{winston2020monotone}. 
We also compare with corresponding explicit versions of our ResNet-based networks given in~\eqref{eq: explicit_network_definition} as well as with state-of-the-art ResNet results~\cite{he2016deep} on the augmented CIFAR10 dataset. 
The explicit networks are trained with the   same setup as their implicit counterparts.
Table~\ref{tab: results} shows JFBs are an effective way to train implicit networks, substantially outperform all the ODE-based networks as well as MONs using similar or fewer parameters.
Moreover, JFB is competitive with Multiscale DEQs~\cite{bai2020multiscale} despite having less than a tenth as many parameters.
Appendix~\ref{app: classification-plots} contains additional results.

\subsection{Comparison to Jacobian-based Backpropagation} \label{subsec: adjoint-comparison}
Table~\ref{tab: adjointComparison} compares   performance between using the standard Jacobian-based backpropagation and JFB.
The experiments are performed on all the datasets described in Section~\ref{subsec: classification}.
To apply the Jacobian-based backpropagation in \eqref{eq: adjoint_equation}, we use the conjugate gradient (CG) method on an associated set of normal equations similarly to~\cite{liao2018reviving}. 
To maintain similar costs, we set the maximum number of CG iterations to be the same as the maximum depth of the forward propagation.
The remaining experimental settings are kept the same as those from our proposed approach (and are therefore not tuned to the best of our ability).
Note the network architectures trained with JFB contain batch normalization in the latent space whereas those trained with Jacobian-based backpropagation do not. 
Removal of batch normalization for the Jacobian-based method was necessary due to a lack of convergence when solving~\eqref{eq: adjoint_equation}, thereby increasing training loss (see Appendix~\ref{app: experimental-settings} for further details). 
This phenomena is also observed in previous works~\cite{bai2020multiscale,bai2019deep}.
Thus, we find JFB to be (empirically) effective on a wider class of network architectures (\eg including batch normalization).
The main purpose of the Jacobian-based results in Figure~\ref{fig: comparison-CIFAR10} and Table~\ref{tab: adjointComparison} is to show speedups in training time while maintaining a competitive accuracy with previous state-of-the-art implicit networks.
More plots are given in Appendix \ref{app: classification-plots}.
 
 \begin{figure}[t]
    \centering
    \includegraphics[width=0.47\textwidth]{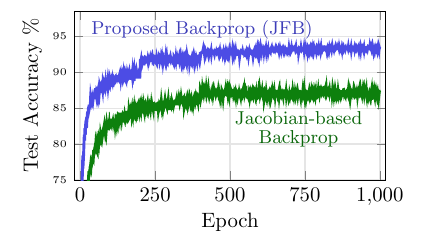}  
    \vspace*{-0.25in}
        
    \includegraphics[width=0.47\textwidth]{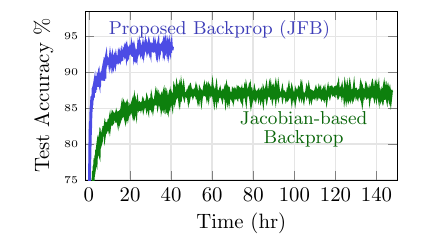}
    \vspace*{-0.3in}    
    \caption{CIFAR10 results using comparable networks/configurations, but with two backpropagation schemes: our proposed JFB method   (blue) and   standard Jacobian-based backpropagation in (\ref{eq: standard-implicit-gradient}) (green), with fixed point tolerance $\epsilon = 10^{-4}$.  JFB is faster and gives better test accuracy.}
    \label{fig: comparison-CIFAR10}
\end{figure}

\subsection{Higher Order Neumann Approximation} \label{subsec: Neumann_comparison}
As explained in Section~\ref{sec: backpropagation}, JFB can be interpreted as an approximation to the Jacobian-based approach by using a truncated series expansion. In particular, JFB is the zeroth order (\ie $k=0$) truncation to  the Neumann series expansion (\ref{eq: neumann-series}) of the Jacobian inverse $\sJ_\Theta^{-1}$. In Figure~\ref{fig: comparison-MNIST-Neumann}, we compare JFB with training that uses more Neumann series terms in the approximation of the the Jacobian inverse $\sJ_\Theta^{-1}$. Figure~\ref{fig: comparison-MNIST-Neumann} shows JFB is competitive at reduced time cost. More significantly, JFB is also much easier to implement as shown in Figure~\ref{fig: pytorch-sample-code}. An additional experiment with SVHN data and discussion about code are provided in Appendix~\ref{app: Neumann_experiments}. \\[20pt]

\begin{figure}[H]
    \centering
        \includegraphics[width = 0.47\textwidth]{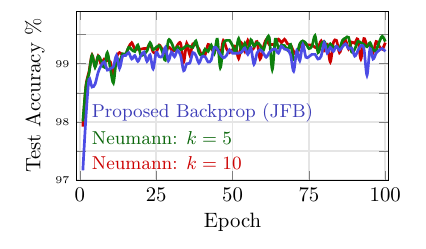} 
        \vspace*{-0.25in}
        
        \includegraphics[width = 0.47\textwidth]{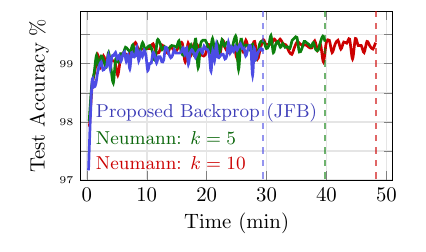}  
        \vspace*{-0.3in}
    \caption{MNIST training using different truncations $k$ of the Neumann series (\ref{eq: neumann-series}) to approximate the inverse Jacobian $\sJ_\Theta^{-1}$. Plots show faster training with fewer terms (fastest with JFB, \ie $k=0$) and competitive test accuracy.}
    \label{fig: comparison-MNIST-Neumann}
\end{figure}

\section{Conclusion} \label{sec: conclusion}
This work presents a new and simple Jacobian-free backpropagation (JFB) scheme.  
JFB enables training of implicit networks with fixed memory costs (regardless of depth), is easy to code (see Figure~\ref{fig: pytorch-sample-code}), and yields efficient backpropagation (by removing computations to do linear solves at each step).
Use of JFB is theoretically justified (even when fixed points are  approximately computed). 
Our experiments show JFB yields competitive results for implicit networks.
Extensions  will enable satisfaction of additional constraints for  imaging and phase retrieval~\cite{klibanov1986determination,fienup1982phase,heaton2020projecting,fung2020multigrid,kan2020pnkh}, geophysics~\cite{haber2014computational,fung2019multiscale,fung2019uncertainty}, and games~\cite{von1959theory,lin2020apac,li2019cubic,ruthotto2020machine}.
Future work will analyze our proposed JFB in stochastic settings.

\section{Acknowledgements}
HH, DM, SO, SWF and QL were supported by 
AFOSR MURI FA9550-18-1-0502 and ONR grants: N00014-18-
1-2527, N00014-20-1-2093, and N00014-20-1-2787.
HH’s work was also supported by the National
Science Foundation (NSF) Graduate Research Fellowship under Grant No. DGE-1650604. Any opinion, findings,
and conclusions or recommendations expressed in this material are those of the authors and do not necessarily reflect
the views of the NSF. 
We thank Zaccharie Ramzi for the fruitful discussions and the anonymous referees for helping us improve the quality of our paper.

{\small
\bibliography{JFB}
}

\clearpage

\onecolumn

\section*{\LARGE{Appendix}}
\appendix


\section{Proofs} \label{app: proofs}

This section provides proofs for     results of Section~\ref{sec: backpropagation}. For the reader's convenience, we restate all results before proving them.

\begin{lemma} \label{lemma: J-coercive}
If Assumption \ref{ass: T-contraction} and \ref{ass: smooth-mappings} hold, then  $\sJ_\Theta$ in (\ref{eq: Jacobian_def}) exists and
\begin{equation}
      \left< u, \sJ_\Theta u\right> \geq (1-\gamma) \|u\|^2,
    \ \ \ \mbox{for all $u\in\sU$.}
    \label{eq: coercivity-J}  
\end{equation}
Additionally, $\sJ_\Theta$ is invertible, and its inverse $\sJ_\Theta^{-1}$ satisfies the coercivity inequality
\begin{equation}
    \left< u, \sJ_\Theta^{-1} u\right> \geq \dfrac{1-\gamma}{(1+\gamma)^2} \|u\|^2,
    \ \ \ \mbox{for all $u\in\sU$.}
    \label{eq: coercivity-J-inverse}
\end{equation}
\end{lemma}
\begin{proof}
We proceed in the following manner.
First we establish  the coercivity inequality (\ref{eq: coercivity-J}) (Step 1). This is used to show $\sJ_\Theta$ is invertible (Step 2).  The previous two results are  then combined to establish the inequality (\ref{eq: coercivity-J-inverse}) (Step 3). All unproven results that are quoted below about operators are standard and may be found   standard functional analysis texts (\eg \cite{kreyszig1978introductory}).\\

{\bf Step 1.} 
To obtain our coercivity inequality, we identify a bound on the operator norm for $\partial T_\Theta/\partial u$.
Fix any unit vector $v\in\sU$. Then, by the definition of differentiation,
\begin{equation}
      \dd{T_\Theta}{u} v   = \lim_{\varepsilon\rightarrow 0^+} \dfrac{T_\Theta(u^\star+\varepsilon v;d)-T_\Theta(u^\star; d)}{ \|(u^\star+\varepsilon v)-u^\star\|}
      = \lim_{\varepsilon\rightarrow 0^+} \dfrac{T_\Theta(u^\star+\varepsilon v;d)-T_\Theta(u^\star; d)}{  \varepsilon}.
    \label{eq: T-Theta-gradient-definition}
\end{equation}
Thus,
\begin{equation} 
    \left\| \dd{T_\Theta}{u} v\right\|
    = \left\| \lim_{\varepsilon\rightarrow 0^+} \dfrac{T_\Theta(u^\star+\varepsilon v;d)-T_\Theta(u^\star; d)}{  \varepsilon} \right\| 
    = \lim_{\varepsilon\rightarrow 0^+}   \dfrac{\|T_\Theta(u^\star + \varepsilon v;d)-T_\Theta(u^\star; d)\|}{  \varepsilon},
\end{equation}
where the first equality follows from (\ref{eq: T-Theta-gradient-definition}) and the second holds  by the continuity of norms. Combining (\ref{eq: T-norm-bound}) with  the Lipschitz assumption (\ref{eq: T-u-gamma-contraction}) gives the upper bound
\begin{equation}
    \left\| \dd{T_\Theta}{u} v\right\|
    \leq \lim_{\varepsilon\rightarrow 0^+}  \dfrac{\gamma \|(u^\star +\varepsilon v) - u^\star \|}{\varepsilon} 
    = \gamma.
    \label{eq: T-norm-bound}
\end{equation}
Because the upper bound relation in (\ref{eq: T-norm-bound}) holds for an arbitrary unit vector $v\in\sU$, we deduce
\begin{equation}
    \left\| \dd{T_\Theta}{u}\right\| \triangleq \sup \left\lbrace \left\| \dd{T_\Theta}{u} v \right\| :\ \|v\| =1 \right\rbrace \leq \gamma.
    \label{eq: T-op-norm-gamma}
\end{equation}
That is, the operator norm  is bounded by $\gamma$.
Together the Cauchy-Schwarz inequality and (\ref{eq: T-op-norm-gamma}) imply
\begin{equation}
    \left< u, \dd{T_\Theta}{u} u \right>
    \leq \left\| \dd{T_\Theta}{u} \right\| \|u\|^2
    \leq \gamma \| u\|^2,
    \ \ \ \mbox{for all}\ u\in\sU.
\end{equation}
Thus, the bilinear form $\left<\ \cdot\ ,\ \sJ_\Theta\  \cdot\ \right>$ is $(1-\gamma)$ coercive, \ie
\begin{equation}
    \left< u, \sJ_\Theta u\right>
    = \|u\|^2 - \left< u, \dd{T_\Theta}{u} u \right>
    \geq (1-\gamma) \|u\|^2,
    \ \ \ \mbox{for all}\ u\in \sU.
    \label{eq: u-Tu-inequality}
\end{equation}

{\bf Step 2.} 
Consider any kernel element $w \in \ker(\sJ_\Theta)$. Then (\ref{eq: u-Tu-inequality}) implies
\begin{equation}
    (1-\gamma)\|w\|^2
    \leq \left< w , \sJ_\Theta w\right>
    = \left< w, 0\right>
    = 0
    \ \ \ \Longrightarrow \ \ \ 
    (1-\gamma)\|w\|^2 \leq 0
    \ \ \ \Longrightarrow \ \ \ 
    w = 0.
\end{equation}
Consequently, the kernel of $\sJ_\Theta$ is trivial, \ie
\begin{equation}
    \ker(\sJ_\Theta) \triangleq \{ u : \sJ_\Theta u = 0 \} = \{0\},
\end{equation}
and wherefore the linear operator $\sJ_\Theta$ is invertible. \\

{\bf Step 3.} 
By (\ref{eq: T-norm-bound}) and an elementary result in functional analysis,  
\begin{equation}
    \|\sJ_\Theta^\top \sJ_\Theta\| = \|\sJ_\Theta\|^2 
    \leq \left( \|\mathrm{I}\| + \left\|\dd{T_\Theta}{u}\right\|\right)^2
    \leq (1+\gamma)^2.
    \label{eq: J-theta-upper-bound}
\end{equation}
Hence
\begin{equation}
    \|u\|^2
    = \left< u, u\right>
    = \left< \sJ_\Theta^{-1} u, (\sJ_\Theta^\top \sJ_\Theta) \sJ_\Theta^{-1} u\right>
    \leq (1+\gamma)^2 \left\| \sJ_\Theta^{-1} u\right\|^2,
    \ \ \ \mbox{for all $u\in\sU$.}
    \label{eq: J-inverse-scalar-product}
\end{equation}
Combining (\ref{eq: u-Tu-inequality})  and (\ref{eq: J-inverse-scalar-product}) reveals
\begin{equation}
    \dfrac{1-\gamma}{(1+\gamma)^2} \left< u ,u\right>
    \leq (1-\gamma) \|\sJ_\Theta^{-1} u\|^2
    \leq \left< \sJ^{-1}_\Theta u, \sJ_\Theta (\sJ_\Theta^{-1} u)\right>
    = \left< \sJ_\Theta^{-1}u, u\right>,
    \ \ \ \mbox{for all $u\in\sU$.}
\end{equation}
This establishes (\ref{eq: coercivity-J-inverse}), and we are done. 
\end{proof}
\newpage

\begin{lemma} \label{lemma: symmetric-matrix-norm}
If $A\in \bbR^{t\times t}$ is symmetric with positive eigenvalues,   
\begin{equation}
    \overline{\lambda} \triangleq \dfrac{\lambda_{\mathrm{max}}(A) + \lambda_{\mathrm{min}}(A)}{2}
    \ \ \ \mbox{and} \ \ \ 
    S \triangleq \overline{\lambda} \mathrm{I} - A, 
    \label{eq: symmetric-matrix-norm}
\end{equation}
then
\begin{equation}
    \|S\| = \dfrac{\lambda_{\mathrm{max}}(A) - \lambda_{\mathrm{min}}(A)}{2}.
\end{equation}
\end{lemma}
\begin{proof}
Since $A$ is symmetric, the spectral theorem asserts it possesses a set of eigenvectors  that form an orthogonal basis for $\bbR^{t}$. This same basis forms the set of eigenvectors for $\overline{\lambda} \mathrm{I} - A$, with eigenvalues of $A$ denoted by $\{\lambda_i\}_{i=1}^{t}$. So, there exists orthogonal $P\in\bbR^{t\times t}$ and diagonal $\Lambda$ with entries given by each of the eigenvalues $\lambda_i$ such that
\begin{equation}
    S =  \overline{\lambda} \mathrm{I} - P^\top \Lambda P
    = P^\top \left( \overline{\lambda}\mathrm{I} - \Lambda\right)P.
\end{equation}
Substituting this equivalence into the definition of the operator norm yields
\begin{equation}
     \| S\| \triangleq \sup \left\lbrace \left\|  S\xi \right\| :\ \|\xi\| =1 \right\rbrace 
    =  \sup \left\lbrace  \|P^\top (\overline{\lambda} I-\Lambda)P\xi\| :\ \|\xi\| =1 \right\rbrace.
    \label{eq: S-norm-1}
\end{equation} 
Leveraging the fact $P$ is orthogonal enables the supremum above to be restated via
\begin{equation}
     \| S\|  
     = \sup \left\lbrace  \|  (\overline{\lambda} I-\Lambda)P\xi\| :\ \|\xi\| =1 \right\rbrace
     = \sup \left\lbrace   \|(\overline{\lambda} I-\Lambda)\zeta\| :\ \|\zeta\| =1 \right\rbrace  .
     \label{eq: S-norm-2}
\end{equation}
Because $\overline{\lambda}\mathrm{I}-\Lambda$ is diagonal, (\ref{eq: S-norm-2}) implies
\begin{equation}
    \| S\| =  \max_{i\in[t]}|\overline{\lambda}-\lambda_i| =   \dfrac{\lambda_{\mathrm{max}}(A) - \lambda_{\mathrm{min}}(A)}{2},
\end{equation} 
and the proof is complete. 
\end{proof}

\newpage
\noindent {\bf Theorem \ref{thm: backprop-descent}.} \textit{If  Assumptions \ref{ass: T-contraction}, \ref{ass: smooth-mappings}, \ref{ass: independent-theta}, and \ref{ass: good-conditioning} hold for given weights $\Theta$ and data $d$,    then
    \begin{equation}
        p_\Theta \triangleq 
        - \dd{}{\Theta}\Big[ \ell(y_d, S_\Theta(T_\Theta(u,d))\Big]_{u=u_d^\star}
        \label{eq: descent-direction-appendix}
    \end{equation} 
    forms a descent direction for $\ell(y_d,\sN_\Theta(d))$ with respect to $\Theta$. }
\begin{proof} 
To complete the proof, it suffices to show
\begin{equation}
    \left< \dd{\ell}{\Theta}, p_\Theta \right> 
    < 0,
    \ \ \ \mbox{for all} \  \ 
    \dd{\ell}{\Theta} \neq 0.
    \label{eq: proof-descent-inequality}
\end{equation} 
Let any weights $\Theta$ and data $d$ be given, and assume the gradient $\mathrm{d}\ell/\mathrm{d}\Theta$ is nonzero. We proceed in the following manner. First we show  $p_\Theta$ is equivalent to a preconditioned gradient (Step 1). We then\\[1.5pt] show  $M^\top \mathrm{d}\ell /\mathrm{d}\Theta$ is nonzero, with $M$ as in (\ref{eq: M-def-gradient}) of Assumption \ref{ass: good-conditioning} (Step 2). These two results are then combined to verify the descent inequality (\ref{eq: proof-descent-inequality}) for the provided $\Theta$ and $d$ (Step 3).\\

\noindent {\bf Step 1.} Denote the dimension of each component of the gradient $\mathrm{d}\ell / \mathrm{d}\Theta$ using\footnote{We assumed each space is a real-valued finite dimensional Hilbert space, making it equivalent to some Euclidean space. So, it suffices to show everything in Euclidean spaces.}
\begin{equation}
    \pp{T_\Theta}{\Theta} \in \bbR^{p\times n},\ \ \  \
    \sJ_\Theta^{-1} \in \bbR^{n\times n}, \ \ \ \ 
    \pp{S_\Theta}{\Theta} \in \bbR^{p\times c}, \ \ \ 
    \dd{S_\Theta}{u} \in \bbR^{n\times c},\ \ \ \
    \pp{\ell}{y} \in \bbR^{c\times 1}.
\end{equation}
Combining each of these terms yields the gradient expression\footnote{In the main text, the ordering was used to make clear application of the chain rule, but here we reorder terms to get consistent dimensions in each matrix operation.}
\begin{equation}
    \dd{\ell}{\Theta}
    = \left[ \pp{T_\Theta}{\Theta}
    \sJ_\Theta^{-1} 
    \dd{S_\Theta}{u}  + \dd{S_\Theta}{\Theta} \right]
    \pp{\ell}{y}.
\end{equation}
By Assumption \ref{ass: independent-theta}, $S_\Theta$ and $T_\Theta$ depend on separate components of $\Theta = (\theta_S,\theta_T)$.
Thus,  
\begin{equation}
    \dd{\ell}{\Theta}
    = \left[ \begin{array}{c} 
    \pp{S_\Theta}{\theta_S}  \\[3pt]
     \pp{T_\Theta}{\theta_T}
    \sJ_\Theta^{-1} 
    \dd{S_\Theta}{u} \end{array} \right]
    \pp{\ell}{y}
    = \underbrace{\left[ \begin{array}{cc}
     \pp{S_\Theta}{\theta_S} & 0  \\
     0 & \pp{T_\Theta}{\theta_T}
    \end{array}\right]}_{M}
     \underbrace{\left[\begin{array}{cc} 
     \mathrm{I} & 0 \\
     0 &  \sJ_\Theta^{-1} 
     \end{array}\right]}_{ \tilde{\sJ}_\Theta^{-1} }
    \underbrace{\left[\begin{array}{c}
     \mathrm{I} \\  \dd{S_\Theta}{u}   
    \end{array}\right]
    \pp{\ell}{y}}_{ v},
\end{equation}
where we define\footnote{Note this choice of $M$ coincides with the matrix $M$ in Assumption \ref{ass: good-conditioning}.} $M \in \bbR^{p\times (n+c)}$, $\tilde{\sJ}_\Theta^{-1}\in \bbR^{(n+c)\times (n+c)}$, and $v\in \bbR^{(n+c)\times 1}$ to be the underbraced quantities. This enables the gradient to be concisely expressed via the relation
\begin{equation}
    \dd{\ell}{\Theta} =  M \tilde{\sJ}_\Theta^{-1}v,
\end{equation}
and our proposed gradient alternative in (\ref{eq: descent-direction-appendix})  is given by
\begin{equation}
    p_\Theta = - Mv.
\end{equation}
Because $M$ has full column rank (by Assumption \ref{ass: good-conditioning}), $M^+ M = \mathrm{I}$,  enabling us to rewrite $p_\Theta$ via
\begin{equation}
    p_\Theta = -M \tilde{\sJ}_\Theta  M^+ M \sJ_\Theta^{-1} v
    = - { (M\tilde{\sJ}_\Theta M^+)} \dd{\ell}{\Theta}.
    \label{eq: p-theta-conjugate-form}
\end{equation}
Hence $p_\Theta$ is a preconditioned gradient ({\em n.b.} the preconditioner  is not necessarily symmetric). \\

\noindent{\bf Step 2.} Set
\begin{equation}
    w \triangleq M^\top \dd{\ell}{\Theta} =  M^\top M \tilde{\sJ}_\Theta^{-1} v.
\end{equation}
The fact that $M$ has full column rank   implies it has a trivial kernel. In particular,
\newcommand{\RA}{\ \ \ \Longrightarrow \ \ \ }
\begin{equation}
    0 \neq \dd{\ell}{\Theta} = M\tilde{\sJ}_\Theta^{-1} v
    \RA
    0 \neq \tilde{\sJ}_\Theta^{-1} v.
    \label{eq: w-neq-0-warm-up}
\end{equation}
Again leveraging the full column rank of $M$, we know $M^\top M$ is invertible and, thus, has trivial kernel as well. This fact together with (\ref{eq: w-neq-0-warm-up}) reveals
\begin{equation}
    0 \neq (M^\top M) \tilde{\sJ}_\Theta^{-1} v = w.
    \label{eq: w-neq-0}
\end{equation}

\noindent{\bf Step 3.} Inserting the definition of $w$ and $p_\Theta$ formulation of (\ref{eq: p-theta-conjugate-form})  into the scalar product in (\ref{eq: proof-descent-inequality}) yields
\begin{equation}
    \left< \dd{\ell}{\Theta}, p_\Theta\right>
    = -\left< M^\top M \tilde{\sJ}_\Theta^{-1} v, \tilde{\sJ}_\Theta  M^+ M \tilde{\sJ}_\Theta^{-1} v \right> \\
    = -\left< w, \tilde{\sJ}_\theta(M^\top M)^{-1}  w\right>,
    \label{eq: proof-key-1}
\end{equation}
noting $M^+ = (M^\top M)^{-1} M^\top$.
Let $\lambda_+$ and $\lambda_-$ be the maximum and minimum eigenvalues of $(M^\top M)^{-1}$, respectively.
Note $(M^\top M)$ is positive definite since the full column rank of $M$   implies
\begin{equation}
    \left< \xi, M^\top M\xi\right>
    = \|M\xi\|^2 > 0,\ \ \ \mbox{for all nonzero $\xi\in\bbR^{n+c}$.}
\end{equation}
Thus, $(M^\top M)^{-1}$ is positive definite,  making $\lambda_+, \lambda_- > 0$. Let $\overline{\lambda}$   be the average of these terms, \ie 
\begin{equation}
    \overline{\lambda} \triangleq \dfrac{\lambda_++\lambda_-}{2}.
\end{equation}
Substituting in this choice of $\overline{\lambda}$ to (\ref{eq: proof-key-1}) by adding and subtracting $\overline{\lambda}\mathrm{I}$ gives the inequality
\begin{equation}
    -\left< w, \tilde{\sJ}_\theta(M^\top M)^{-1}  w\right>
   \leq - \overline{\lambda} (1-\gamma)\|w\|^2 + \left< w , \tilde{\sJ}_\Theta (\overline{\lambda}\mathrm{I} - (M^\top M)^{-1}) w\right>,
   \label{eq: proof-product-inequality}
\end{equation}
noting $\tilde{\sJ}_\Theta$ is $1-\gamma$ coercive because it is the block diagonal composition of $\sJ_\Theta$, which is $1-\gamma$ coercive by (\ref{eq: coercivity-J}  ) in Lemma \ref{lemma: J-coercive}, and the identity matrix, which is 1-coercive.
Application of the Cauchy Schwarz inequality to the right hand side of (\ref{eq: proof-product-inequality}) reveals
\begin{equation}
    -\left< w, \tilde{\sJ}_\theta(M^\top M)^{-1}  w\right> 
    \leq -\overline{\lambda} (1-\gamma)\|w\|^2 +  \|\tilde{\sJ}_\Theta\|\|\overline{\lambda} \mathrm{I} - (M^\top M)^{-1})\|\|w\|^2.
    \label{eq: proof-key-2}
\end{equation}
By Lemma \ref{lemma: symmetric-matrix-norm},
\begin{equation}
    \|\overline{\lambda} \mathrm{I} - (M^\top M)^{-1}\| = \dfrac{\lambda_+ - \lambda_-}{2} .
    \label{eq: proof-key-3}
\end{equation} 
Similar block diagonal argument as used above to verify $\tilde{\sJ}_\Theta$ is coercive can also be applied to bound the operator norm of $\tilde{\sJ}_\Theta$. Indeed, (\ref{eq: T-op-norm-gamma}) implies 
\begin{equation}
    \|\sJ_\Theta\| \leq 1+ \gamma
    \ \ \ \Longrightarrow \ \ \ 
     \|\tilde{\sJ}_\Theta\| \leq 1+\gamma.
     \label{eq: proof-key-4}
\end{equation}
Hence   (\ref{eq: proof-key-1}), (\ref{eq: proof-key-2}),   (\ref{eq: proof-key-3}), and (\ref{eq: proof-key-4})  together yield
\begin{equation}
     \left< \dd{\ell}{\Theta}, p_\Theta\right>
     \leq -\dfrac{1}{2} \big( (1-\gamma)(\lambda_+ + \lambda_-) -   (1+\gamma)(\lambda_+-\lambda_-)\big) \|w\|^2
     = -2(\lambda_- - \gamma \lambda_+) \|w\|^2.
     \label{eq: proof-key-5}
\end{equation}
The right hand expression in (\ref{eq: proof-key-5}) is negative since (\ref{eq: w-neq-0}) shows $w\neq 0$ and the conditioning inequality (\ref{eq: gradient-conditioning}) in Assumption \ref{ass: good-conditioning} implies  $(\lambda_- - \gamma \lambda_+)$ is positive. This verifies (\ref{eq: proof-descent-inequality}), completing the proof.
\end{proof}

\newpage 
{\bf Corollary \ref{cor: perturbation-resilience}.}
\textit{Given weights $\Theta$ and data $d$, there exists $\varepsilon > 0$ such that
if $u^\varepsilon\in\sU$ satisfies $\|u_d^\varepsilon- u_d^\star\| \leq \varepsilon$ and the assumptions of Theorem \ref{thm: backprop-descent} hold, then  
    \begin{equation}
        p_\Theta^\varepsilon \triangleq -- \dd{}{\Theta}\Big[ \ell(y_d, S_\Theta(T_\Theta(u,d))\Big]_{u=u_d^\varepsilon}
    \end{equation} 
    is a descent direction for the loss function $\ell(y_d, \sN_\Theta(u_d^\star,d))$ with respect to $\Theta$.   }
 \begin{proof}
 For notational convenience, for all $\tilde{u}\in\sU$, define
 \begin{equation}
     p_\Theta(\tilde{u}) \triangleq - \dd{}{\Theta}\Big[ \ell(y_d, S_\Theta(T_\Theta(u,d))\Big]_{u=\tilde{u}}
 \end{equation}
 noting $p_\Theta^\varepsilon = p_\Theta(u_d^\varepsilon)$.
 Also define the quantity
 \begin{equation}
    \nabla 
    \triangleq   \dd{}{\Theta}\left[ \ell(y_d, \sN_\Theta(  d)) \right].
 \end{equation}
 Assuming $\nabla \neq 0$, it suffices to show
 \begin{equation}
     \left<  p_\Theta^\varepsilon , \nabla  \right> < 0.
     \label{eq: corollary-proof-0}
 \end{equation}
 By the smoothness of $\ell$, $S_\Theta$, and $T_\Theta$ (see Assumption \ref{ass: smooth-mappings}), there exists $\delta > 0$ such that
 \begin{equation}
     \|u - u_d^\star\| \leq \delta
     \ \ \ \Longrightarrow \ \ \ 
     \left\| p_\Theta(u) - p_\Theta(u_d^\star) \right\|
     \leq \dfrac{(\lambda_- - \gamma \lambda_+) \|M^\top \nabla\|^2}{\|\nabla\|},
     \label{eq: corollary-proof-1}
 \end{equation}
 where $\lambda_+$ and $\lambda_-$ are the maximum and minimum eigenvalues of $(M^\top M)^{-1}$, respectively. Also note $M^\top \nabla \neq 0$ since $M^\top$ has full column rank.\footnote{See $w$ in Step 2 of the proof of Theorem \ref{thm: backprop-descent}.}
 Substituting   the inequality (\ref{eq: proof-key-5}) in the proof of Theorem \ref{thm: backprop-descent}  into (\ref{eq: corollary-proof-0}) reveals
 \begin{subequations}
 \begin{align}
     \left<  p_\Theta(u) , \nabla \right>
     &= \left< p_\Theta(u_d^\star), \nabla\right> + \left< p_\Theta(u) - p_\Theta(u_d^\star), \nabla \right>\\
     &\leq -2(\lambda_- - \gamma \lambda_+) \|M^\top \nabla \|^2  + \left< p_\Theta(u) -p_\Theta(u_d^\star), \nabla \right>.      
     \end{align}
     \label{eq: corollary-proof-3}
 \end{subequations}
 But, the Cauchy Schwarz inequality and (\ref{eq: corollary-proof-1}) enable us to obtain the upper bound
 \begin{equation}
     |\left< p_\Theta(u) - p_\Theta(u_d^\star), \nabla \right>|
     \leq (\lambda_- - \gamma \lambda_+) \|M^\top \nabla \|^2 ,
     \ \ \ \mbox{for all $u \in B(u_d^\star,\delta)$,}
    \label{eq: corollary-proof-2}
 \end{equation}
 where $B(u_d^\star,\delta)$ is the ball of radius $\delta$ centered about $u_d^\star$.
  Combining (\ref{eq: corollary-proof-3}) and (\ref{eq: corollary-proof-2}) yields
  \begin{equation}
     \left<  p_\Theta(u) , \nabla \right> \leq -(\lambda_- - \gamma \lambda_+) \|M^\top \nabla \|^2,
      \ \ \  \mbox{for all $u \in B(u_d^\star,\delta)$.}
  \end{equation}
  In particular, this shows (\ref{eq: corollary-proof-0}) holds when we set $\varepsilon=\delta$. 
\end{proof}

\section{Classification Accuracy Plots} \label{app: classification-plots}

\begin{figure}[H]
    \centering
        \includegraphics[width = 0.47\textwidth]{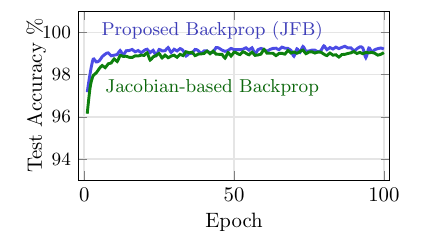} 
        \hspace{-20pt}
        \includegraphics[width = 0.47\textwidth]{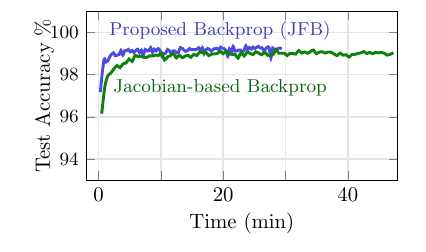}
        \vspace*{-8pt}
    \caption{MNIST performance using nearly identical  architectures/configurations, but with two backpropagation schemes: our proposed method   (blue) and the standard Jacobian-based backpropagation in (\ref{eq: standard-implicit-gradient}) (red), with fixed point tolerance $\epsilon = 10^{-4}$. The difference in the  architecture/configurations comes from the use of batch normalization in the latent space when using JFB (see Appendix~\ref{app: experimental-settings} for more details).
    Our method is faster and yields better test accuracy.}
    \label{fig: comparison-MNIST}
\end{figure}

\begin{figure}[H]
    \centering
        \includegraphics[width = 0.47\textwidth]{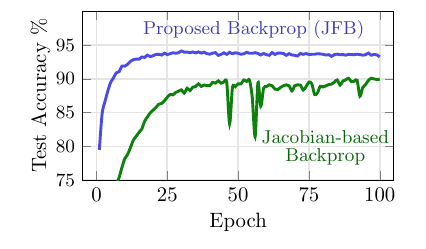} 
        \hspace{-20pt}
        \includegraphics[width = 0.47\textwidth]{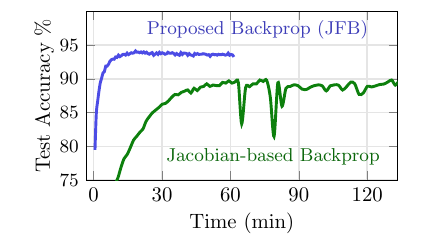}
        \vspace*{-8pt}
    \caption{SVHN performance using identical  architecture/configurations, but with two backpropagation schemes: our proposed method   (blue) and the standard Jacobian-based backpropagation in (\ref{eq: standard-implicit-gradient}) (red), with fixed point tolerance $\epsilon = 10^{-4}$.  
    The difference in the  architectures/configurations comes from the use of batch normalization in the latent space when using JFB (see Appendix~\ref{app: experimental-settings} for more details).
    Our method is faster and yields better test accuracy.}
    \label{fig: comparison-SVHN}
\end{figure}

\section{Implementation of Jacobian-based Backpropagation}
\label{app: Jacobian_implementation_details}
\subsection{Implementation Notes}
In this section, we provide some notes to help understand the code/implementation of the Jacobian-based backpropagation in PyTorch.
Assume we have the fixed point $\tilde{u}$ at hand. For brevity, we will omit the dependence of $R_\Theta$ and $\tilde{u}$ on $d$.
We wish to compute 
\begin{align}
    \dd{}{\Theta}\Big[ \ell(y_d, S_\Theta(R_\Theta(\tilde{u}_\Theta))\Big] 
    &= \dd{\ell}{S} \left[ \dd{S}{\tilde{u}} \dd{\tilde{u}}{\Theta} + \frac{\partial S}{\partial \Theta} \right],
\end{align}
where
\begin{equation}
    \dd{\tilde{u}}{\Theta} = \sJ^{-1} \dd{R_\Theta(\tilde{u})}{\Theta},
\end{equation}
and the argument $\tilde{u}$ inside of $R$ is treated as a constant.
This implies that 
\begin{equation}
    \dd{}{\Theta}\Big[ \ell(y_d, S_\Theta(R_\Theta(\tilde{u}_\Theta))\Big] 
    = 
    \dd{\ell}{S} \left[ \dd{S}{\tilde{u}} \sJ^{-1} \dd{R_\Theta(\tilde{u})}{\Theta} + \frac{\partial S}{\partial \Theta} \right].
\end{equation}
In our PyTorch implementation, we do not build $J^{-1}$ explicitly. Instead, we solve a linear system as follows. We would like to compute $w$ defined by
\begin{equation}
    w = \dd{\ell}{S} \dd{S}{\tilde{u}} \sJ^{-1}.
\end{equation}
To do this, we solve the following linear system
\begin{equation}
    w \sJ = \dd{\ell}{S} \dd{S}{\tilde{u}}
\end{equation}
Note, we consider multiplication by matrices from the right as this is more natural to implement in PyTorch. We also note that building the matrix $\sJ$ explicitly is inefficient, thus any matrix-factorization methods (\eg the $LU$ decomposition) cannot be used.
As explained in Section~\ref{sec: experiments}, we use a CG method and require a symmetric coefficient matrix. To this end, we symmetrize the system by multiplying by $J^\top$ on both sides to obtain the normal equations~\cite{golub2013matrix}
\begin{equation}
    \label{eq: normal_equations}
    w \sJ \sJ^\top = \dd{\ell}{S} \dd{S}{\tilde{u}} \sJ^\top.
\end{equation}
Once we solve for $w$, we can then arrive at the gradient by computing 
\begin{equation}
    \dd{}{\Theta}\Big[ \ell(y_d, S_\Theta(R_\Theta(\tilde{u}_\Theta))\Big] = w \dd{R_\Theta(\tilde{u})}{\Theta} + \dd{\ell}{S} \frac{\partial S}{\partial \Theta}.
\end{equation}
\ \\[0.5in]
 
\subsection{Coding right-hand-side}
To code the right-hand-side of the normal equations, we can code $\dd{\ell}{S} \dd{S}{\tilde{u}}$ in the following line of code:

\begin{tcolorbox}[colback=white,colframe=black!30!white,title=Computing $\dd{\ell}{S} \dd{S}{\tilde{u}}$, coltitle=black, left=1mm]
\vspace*{-5pt}
\begin{verbatim}  
Qd = net.data_space_forward(d)
Ru = net.latent_space_forward(u, Qd)
S_Ru = net.map_latent_to_inference(Ru)
loss = criterion(S_Ru, labels)
dldu = torch.autograd.grad(outputs=loss, inputs=Ru,
                           retain_graph=True, create_graph=True,
                           only_inputs=True)[0]
\end{verbatim} 
\end{tcolorbox}
Next, we would like to multiply \texttt{dldu} by $J^\top$ from the right side. To do this, we need to use a vector-Jacobian trick in Pytorch as follows:
\begin{tcolorbox}[colback=white,colframe=black!30!white,title=Computing $\dd{\ell}{S} \dd{S}{\tilde{u}} J^\top$, coltitle=black, left=1mm]
\vspace*{-5pt}
\begin{verbatim}  
dldu_dRdu = torch.autograd.grad(outputs=Ru, inputs = u, grad_outputs=dldu, 
                                retain_graph=True, create_graph=True,
                                only_inputs=True)[0]
dldu_J = dldu - dldu_dRdu

dldu_JT = torch.autograd.grad(outputs=dldu_J, inputs=dldu, grad_outputs=dldu, 
                              retain_graph=True, create_graph=True,
                              only_inputs=True)[0]

rhs = dldu_JT
\end{verbatim} 
\end{tcolorbox}
Here, to multiply by $J^\top$ from the right, we note that for any vector $v$, 
\begin{equation}
    \label{eq: autograd_trick}
    v \dd{(vJ)}{v} = v J^\top.
\end{equation}
The vector-Jacobian trick uses autograd once to compute $vJ$, and then autograd once more compute $vJ^\top$ as in Equation~\eqref{eq: autograd_trick}.
Thus, we have that \texttt{rhs} takes the value of $\dd{\ell}{S} \dd{S}{\tilde{u}} J^\top$.

\subsection{Coding right matrix-vector multiplication by $J J^\top$}
Next, we want to implement a function that computes right matrix-vector multiplication by $J J^\top$. This function, along with the right-hand-side, is then fed into the conjugate gradient algorithm to solve Equation~\eqref{eq: normal_equations}.

Given a vector $v$, the task is to return $v J J^\top$. First, we use one autograd call to obtain $vJ$. Then we use another autograd call to multiply by $J^\top$ to obtain $vJJ^\top$. The function which multiplies by $JJ^\top$ from the right can thus be coded as
\begin{tcolorbox}[colback=white,colframe=black!30!white,title= Computing multiplication by $JJ^\top$, coltitle=black, left=1mm]
\vspace*{-5pt}
\begin{verbatim}  
v_dRdu = torch.autograd.grad(outputs=Ru, inputs=u, grad_outputs=v, 
                             retain_graph=True, create_graph=True,
                             only_inputs=True)[0]
v_J = v - v_dRdu

v_JJT = torch.autograd.grad(outputs=v_J, inputs=v, grad_outputs=v_J, 
                            retain_graph=True, create_graph=True,
                            only_inputs=True)[0]
\end{verbatim} 
\end{tcolorbox}
We emphasize here that the third line returns $vJJ^\top$ by setting the variable \texttt{grad\_outputs} to be $vJ$.
Finally, we feed the computed right-hand-side and the function that multiplies by $JJ^\top$ into the conjugate gradient method to solve for $w$ in Equation~\eqref{eq: normal_equations}.
\subsection{Coding $w \dd{R_\Theta(\tilde{u})}{\Theta}$ and $\dd{\ell}{S} \frac{\partial S}{\partial \Theta}$}
Once $w$ is obtained from the linear solve, we have two remaining tasks to obtain the gradient: computation of $w \dd{R_\Theta(\tilde{u})}{\Theta}$ and $\dd{\ell}{S} \frac{\partial S}{\partial \Theta}$.
These can be computed as follows in the PyTorch framework. Suppose the solution to the normal equations is saved in the variable \texttt{normal\_eq\_sol}
\begin{tcolorbox}[colback=white,colframe=black!30!white,title= Update gradients, coltitle=black, left=1mm]
\vspace*{-5pt}
\begin{verbatim}  
Ru.backward(normal_eq_sol) 
S_Ru = net.map_latent_to_inference(Ru.detach())
loss = criterion(S_Ru, labels)
loss.backward()    
\end{verbatim}
\end{tcolorbox}
This is only one (perhaps the most straightforward) way to code the Jacobian-based backpropagation. But as can be seen, coding the Jacobian-based backpropagation is not trivial, unlike our proposed JFB.

\section{Comparison with Neumann RBP}
\label{app: Neumann_experiments}
Below is a comparison of JFB with 5th and 10th order Neumann series approximations of   gradients   for the SVHN dataset. 

\begin{figure}[H]
    \centering
    \includegraphics[width=0.47\textwidth]{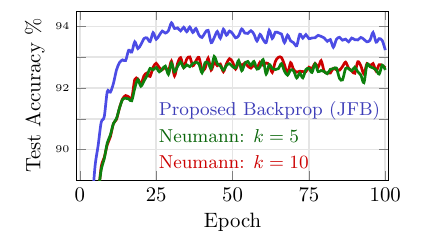}  
    \includegraphics[width=0.47\textwidth]{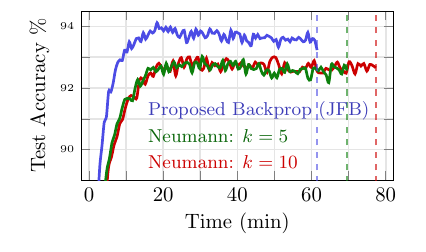}
    \caption{SVHN using different Neumann approximations of the inverse Jacobian.}
    \label{fig: comparison-SVHN-Neumann}
\end{figure}

\subsection{Neumann Gradient Implementation}
To compute the Neumann-based gradient, we use a similar approach that explained in Appendix~\ref{app: Jacobian_implementation_details}. In particular, we use a for-loop to accumulate the Neumann sum in the variable \texttt{dldu\_Jinv\_approx}.
\begin{tcolorbox}[colback=white,colframe=black!30!white,title= Computing Neumann gradient, coltitle=black, left=1mm]
\begin{verbatim}
for i in range(1, neumann_order+1):

    dldu_dRdu_k.requires_grad = True

    # compute dldu_dRdu_k+1 = dldu_dRdu_k * dRdu
    dldu_dRdu_kplus1 = torch.autograd.grad(outputs=Ru,
                                           inputs=u,
                                           grad_outputs=dldu_dRdu_k,
                                           retain_graph=True,
                                           create_graph=True,
                                           only_inputs=True)[0]

    dldu_Jinv_approx = dldu_Jinv_approx + dldu_dRdu_kplus1.detach()

    dldu_dRdu_k = dldu_dRdu_kplus1.detach()

Ru.backward(dldu_Jinv_approx)
\end{verbatim} 
\end{tcolorbox}
Similar to the Jacobian-based approach, we multiply from right by $\frac{\partial R}{\partial \Theta}$ from the right in the last line.


\section{Experimental Settings} \label{app: experimental-settings}
We present the experimental settings and describe the architecture used for each dataset. 
We used ResNets with batch normalization in the latent space portion of the networks, \ie, $R_{\Theta}(d)$. 
While batch normalization prevents us from completely guaranteeing the network is $\gamma$-contractive in its latent variable, we found the networks automatically behave in a contractive manner. Specifically, every time the network is evaluated during training, we check whether our network violates the $\gamma$-contractive property and print a warning when this is the case. This warning was never called in our experiments.
As mentioned in~\eqref{subsec: adjoint-comparison}, the Jacobian-based version failed to converge (even with tighter tolerance and more iterations) when batch normalization was present in the latent space - this is an issue also observed in other implicit networks literature~\cite{bai2020multiscale}. Consequently, we remove the batch normalization for the Jacobian-based runs.
We train all of our networks with the Adam optimizer~\cite{kingma2015adam} and use the cross entropy loss function.

\subsection{MNIST}
We use two convolutions with leaky relu activation functions and max pooling for the data-space portion of the network $Q_{\Theta}(d)$.
In the latent space portion, $R_{\Theta}(d)$, we use 2-layer ResNet-based architecture, with the ResNet block containing two convolution operators with batch normalization. 
Finally, we map from latent space to inference space using one convolution and one fully connected layer.
For the fixed point stopping criterion, we stop whenever consecutive iterates satisfy $\|u^{k+1} - u^{k}\| <\epsilon=10^{-4}$ or 50 iterations have occurred.
We use a constant learning rate of $10^{-4}$.

\subsection{SVHN}
We use three 1-layer ResNets with residual blocks containing two convolutions with leaky relu activation functions and max pooling for the data-space portion of the network $Q_{\Theta}(d)$.
Similarly to the ResNet-based network in MNIST, we use a ResNet block containing two convolution operators with batch normalization in the latent space portion $R_{\Theta}(d)$.
We map from latent space to inference space using one convolution and one fully connected layer.
For the fixed point stopping criterion, we stop whenever consecutive iterates satisfy $\|u^{k+1} - u^{k}\| <\epsilon=10^{-4}$ or 200 iterations have occurred.
We use constant learning rate of $10^{-4}$ with weight decay of $2 \times 10^{-4}$. 
\subsection{CIFAR10}

We use a ResNet with residual blocks containing two convolutions for the data-space portion of the network $Q_{\Theta}(d)$.
We use a ResNet for the latent space portion of $R_{\Theta}(d)$, with each ResNet block containing two convolution operators with batch normalization.
Approximately $70\%$ of the weights are in $Q_\Theta$ and $30\%$ of the weights are in $R_\Theta$. 
We map from latent space to inference space using one convolution and one fully connected layer. 
For the JFB fixed point stopping criterion, we stop whenever consecutive iterates satisfy $\|u^{k+1} - u^{k}\| <\epsilon=10^{-1}$ or 50 iterations have occurred.
For the Jacobian-based approach, however, we observed that we needed to tighten the tolerance in order for the gradients to be computed accurately. Particularly, we stop whenever consecutive iterates satisfy $\|u^{k+1} - u^{k}\| <\epsilon=10^{-4}$ or 500 iterations have occurred.

\section{Toy Implicit Example} \label{app: toy_example}

This section provides rigorous justification of the toy example provided in Section \ref{sec: why-implicit} for solving $y=d+y^5$ with a given  $d\in[-1/2,1/2]$. See Figure \ref{fig: toy example} for an illustration. We first outline its implicit solution in the following lemma, which also establishes this equation has a unique solution in $[-10^{-1/4},10^{-1/4}]$.  This is followed by a brief discussion of the explicit series representations of solutions to (\ref{eq: toy_manifold}).

\begin{figure}[H]
    \centering
    \includegraphics[width=.5\textwidth]{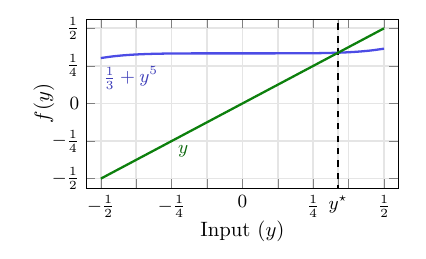}
    \caption{Plot of the functions $y$ and $d+y^5$ with $d=1/3$. }
    \label{fig: toy example}
\end{figure}

\begin{lemma}
    Let $d \in [-1/2,1/2]$. If the sequence $\{y_k\} \subset \bbR$ is defined such that
      $y_1 = 0$ and 
    \begin{equation}
        y_{k+1} = T(y_k;d) \triangleq d + y_k^5
    \end{equation}
    then $\{y_k\}$ converges to the unique fixed point of $T(\cdot;d)$  among $y\in  [-10^{-1/4},10^{-1/4}]$.
\end{lemma}
\begin{proof}
    We proceed in the following manner. First  $T(\cdot;d)$ is shown to a contraction on a restricted subset of $\bbR$ (Step 1). Then we show $\{y_k\}$ is a subset of this restricted subset (Step 2). These two facts together enable us to obtain convergence (Step 3) and uniqueness (Step 4), using a special case of Banach's fixed point theorem \cite{banach1922operations}. \\
    
    \noindent {\bf Step 1.} Set $\alpha = 10^{-1/4}$ and let $y,\gamma \in [-\alpha,\alpha]$. By the mean value theorem, there exists $\xi$  between $y$ and $\gamma$ such that
    \begin{equation}
        |y^5 - \gamma^5|
        = |T(y;d) - T(\gamma; d) |
        = \left|\dd{T}{u}(\xi;d)\right| | y - \gamma|.
    \end{equation}
    Additionally, 
    \begin{equation}
        \sup_{\xi \in [-\alpha,\alpha]} \left| \dd{T}{u}(\xi;d)\right|
        = \sup_{\xi \in [-\alpha,\alpha]} 5y^4
        = 5\alpha^4
        \leq 5 \cdot   \dfrac{1}{10} 
        = \dfrac{1}{2},
    \end{equation}
    and so
    \begin{equation}
        |y^5 - \gamma^5| \leq \dfrac{1}{2}|y-\gamma|.
    \end{equation}
    Because $y$ and $\gamma$ were arbitrarily chosen in $[-\alpha,\alpha]$, it follows that the restriction of $T(\cdot;d)$ to $[-\alpha,\alpha]$ is a $\frac{1}{2}$-contraction.\\
    
    \noindent {\bf Step 2.} This step proceeds by induction. Note $y_1 = 0 \in [-\alpha,\alpha]$.
    Inductively, suppose $y_k \in \bbN$. This implies
    \begin{equation}
        |y_{k+1}|
        = |T(y_k;d)|
        = |d + y_k^5|
        \leq |d| + |y_k|^5
        \leq  \dfrac{1}{2} + \alpha^5
        < \alpha,
    \end{equation}
    and so $y_{k+1} \in [-\alpha,\alpha]$. By the principle of mathematical induction, we deduce $y_k \in [-\alpha,\alpha]$ for all $k\in\bbN$. \\ 
    
    \noindent {\bf Step 3.} We now establish convergence. Applying the results of Step 1 and Step 2 reveals
    \begin{equation}
        |y_{k+2}-y_{k+1}| = |T(y_{k+1}) - T(y_{k})| \leq \dfrac{1}{2}|y_{k+1}-y_k|, \ \ \ \mbox{for all $k\in\bbN$.}
    \end{equation}
    Applying this result inductively with the triangle inequality reveals $m > n$ implies
    \begin{equation}
        |y_m - y_n| \leq \sum_{\ell=n}^{m-1} |y_{\ell+1} - y_\ell|
        \leq \sum_{\ell=n}^{m-1} 2^{-\ell} |y_{2} - y_1|
        \leq 2^{-n} |y_{2} - y_1| \cdot \sum_{\ell=0}^{\infty} 2^{-\ell}  
        \leq 2^{1-n}  |y_2-y_1|.
        \label{eq: toy-lemma-proof-1}
    \end{equation}
    Since the right hand side in (\ref{eq: toy-lemma-proof-1}) converges to zero as $n\rightarrow\infty$, we see $\{y_k\}$ is Cauchy and, thus, converges to a limit $y_\infty$. Moreover, the limit satisfies
    \begin{equation}
        y_\infty 
        = \lim_{k\rightarrow\infty} y_k 
        = \lim_{k\rightarrow\infty} T(y_k; d)
        = \lim_{k\rightarrow\infty} d + y_k^5
        = d + y_\infty^5.
    \end{equation}
    
    \noindent {\bf Step 4.} All that remains it to verify the fixed point of $T(\cdot;d)$ is unique over $[-\alpha,\alpha]$. If a fixed point $\tilde{y} \in \fix{T(\cdot;d)}$ were to exist such that $\tilde{y}\in [-\alpha,\alpha]-\{ y_\infty\}$, then the contractive property of $T(\cdot;d)$ may be applied to deduce
    \begin{equation}
        |y_\infty - \tilde{y}|
        = |T(y_\infty;d) - T(\tilde{y};d)|
        \leq \dfrac{1}{2}| y_\infty - \tilde{y}| 
        \ \ \implies \ \ 
        1 < \frac{1}{2},
    \end{equation}
    a contradiction. Hence the fixed point $y_\infty$ is unique.
\end{proof}

\subsection{Explicit Solution}
As is well-known, the solution of a quintic equation cannot be expressed as a function of the coefficients using only the operations of addition, subtraction, multiplication, division and taking roots \cite{abel1826demonstration}. The simplest way to express the unique root to $y=d+y^5$ lying in the interval $[-10^{-1/4}, 10^{-1/4}]$ as a function of $d$ is via a hypergeometric series by writing
\begin{equation}
y = d\left[_4F_3\left(\frac{1}{5},\frac{2}{5},\frac{3}{5},\frac{4}{5};\frac{1}{2},\frac{3}{4},\frac{5}{4};\frac{3125 d^4}{256}\right) \right]  = d + d^5 + 10\frac{d^9}{2!} + 210 \frac{d^{13}}{3!} + \ldots  
\end{equation}
See \cite{birkeland1927auflosung} or \cite{ottem2011solving} for further information on solving quintic equations using hypergeometric functions. 
\end{document}